\documentclass[letterpaper]{article} 
\usepackage{aaai2026}  
\usepackage{times}  
\usepackage{helvet}  
\usepackage{courier}  
\usepackage[hyphens]{url}  
\usepackage{graphicx} 
\usepackage{natbib}  
\usepackage{caption} 
\usepackage{algorithm}
\usepackage{algorithmic}

\usepackage[final]{pdfpages}

\usepackage{newfloat}
\usepackage{listings}
\DeclareCaptionStyle{ruled}{labelfont=normalfont,labelsep=colon,strut=off} 
\floatstyle{ruled}
\newfloat{listing}{tb}{lst}{}
\floatname{listing}{Listing}
\usepackage{amsmath}
\usepackage{amssymb}
\usepackage{mathtools}
\usepackage{amsthm}
\usepackage{booktabs}
\usepackage{multirow} 

\theoremstyle{plain}
\newtheorem{theorem}{Theorem}

\theoremstyle{definition}

\theoremstyle{remark}

\nocopyright

\title{SineLoRA$\Delta$: Sine-Activated Delta Compression}
\author {
    Cameron Gordon\equalcontrib \textsuperscript{\rm 1},
    Yiping Ji\equalcontrib \textsuperscript{\rm 1},
    Hemanth Saratchandran\equalcontrib \textsuperscript{\rm 1},
    Paul Albert \textsuperscript{\rm 2},
    Simon Lucey \textsuperscript{\rm 1}
}
\affiliations {
    \textsuperscript{\rm 1}Australian Institute for Machine Learning (AIML), University of Adelaide, Australia\\
    \textsuperscript{\rm 2}Amazon Machine Learning, Australia\\

    \{cameron.gordon, yiping.ji, hemanth.saratchandran, simon.lucey\}@adelaide.edu.au
}

\begin{document}

\maketitle

\begin{abstract}

    Resource-constrained weight deployment is a task of immense practical importance. Recently, there has been interest in the specific task of \textit{Delta Compression}, where parties each hold a common base model and only communicate compressed weight updates. However, popular parameter efficient updates such as Low Rank Adaptation (LoRA) face inherent representation limitations - which are especially pronounced when combined with aggressive quantization. To overcome this, we build on recent work that improves LoRA representation capacity by using fixed-frequency sinusoidal functions to increase stable rank without adding additional parameters. We extend this to the quantized setting and present the first theoretical analysis showing how stable rank evolves under quantization. From this, we introduce \textbf{SineLoRA$\Delta$}, a principled and effective method for delta compression that improves the expressivity of quantized low-rank adapters by applying a sinusoidal activation. We validate SineLoRA$\Delta$ across a diverse variety of domains - including language modeling, vision-language tasks, and text-to-image generation - achieving up to 66\% memory reduction with similar performance. We additionally provide a novel application of the canonical Bjøntegaard Delta metric to consistently compare adapter compression changes across the rate-distortion curve.
    

\end{abstract}

\section{Introduction}

Parameter-Efficient Fine-Tuning (PEFT) has emerged as a core component of modern machine learning pipelines \cite{pmlr-v97-houlsby19a,han2024parameterefficient}. Most PEFT methods adapt a frozen pre-trained backbone by learning a small set of task-specific parameters, often implemented as additive weight updates. Among these approaches, Low-Rank Adapters have become especially prominent, with a rapidly expanding literature \cite{hu2022lora,mao2025survey}. Recent work has sought to further reduce the number of trainable parameters by exploring alternative low-rank decompositions \cite{karimi_2021, edalati2022krona, liu2024dora, he2023parameter, ding2023sparse, albert2025randlora, kopiczko2024vera, koohpayegani2024nola}.

Recently, a new fine-tuning paradigm has emerged that enhances the expressive power of low-rank adapters by applying rank-enhancing functions component-wise. Introduced in \cite{ji2025efficient}, this approach demonstrates that applying a non-linear transformation, specifically, a fixed-frequency sinusoidal function, to a low-rank adapter can significantly increase its rank. This expressivity gain comes at no additional parameter cost, preserving the memory efficiency of LoRA while yielding higher-rank representations.

In this paper, we investigate the interaction between rank-enhancing sinusoidal non-linearities and quantization, a technique that maps full-precision parameters to a smaller set of discrete values, ideally with minimal impact on model performance \cite{han2015deep, gholami2021surveyquantizationmethodsefficient, min_2024}. Quantization is a key enabler for deploying large models on resource-constrained hardware, offering improvements in memory efficiency, computational throughput, and energy consumption \cite{gholami2021surveyquantizationmethodsefficient, Dettmers_QLORA_23, xu2024qalora, kaushal2025surprising}.

To study this interaction, we develop a theoretical framework that characterizes how the rank of an adapter changes under quantization, showing that it is tightly controlled by the rank of the original, unquantized adapter. This leads to our key insight: when the adapter has low-rank, as is the case with LoRA, quantization preserves this structure. However, by applying a component-wise sinusoidal non-linearity after quantization, we can enrich the representational capacity of the adapter, effectively compensating for the rank limitation and enabling more expressive quantized models.

This insight is particularly relevant in the context of adapter quantization, which has emerged as one of two dominant approaches in quantized fine-tuning. The first, exemplified by QLoRA \cite{Dettmers_QLORA_23, badri2023hqq}, applies quantization to the base model while maintaining high-precision adapters and activations. This approach is primarily motivated by reducing memory overhead during fine-tuning, making it feasible to adapt large language models on a single GPU \cite{Dettmers_QLORA_23}. The second approach focuses on quantizing the adapters themselves \cite{yao2023deltazip, liu2024bitdelta, isik2023gptzip, ping2024deltacome, Jie_2023}, enabling highly compact and transferable fine-tuned models. Our work follows this latter direction, showing that rank-enhancing sinusoidal functions can be easily integrated as a plug-in component into quantized adapters, significantly improving their expressivity while retaining the memory efficiency that makes adapter quantization attractive. 

To our knowledge, such rank-enhancing functions have not yet been explored in the quantization literature, and we view this work as a first step towards bridging that gap.

\noindent
Our main contributions are as follows:
\begin{itemize}
\item We provide the first theoretical analysis showing how quantization affects the stable rank of a fine-tuning adapter, and show that this rank is tightly governed by the rank of its unquantized counterpart.
\item Based on our theoretical results we demonstrate that the effects of quantization on rank can be mitigated by applying rank enhancing functions in the form of sinusoids with fixed frequencies.
\item In order to consistently evaluate compression effectiveness, we introduce a novel application of Bjøntegaard Delta (canonically used to compare image codecs) to compare the effect of PEFT compression approaches.
\end{itemize}

We evaluate our approach through extensive experiments on vision and language tasks, including Large Language Model adaptation, Vision-Language Model Adaptation, and Text-to-Image Generation. For evaluation on Commonsense Reasoning, we show that memory reductions of up to 66\% is achievable relative to full-precision LoRA models by combining quantization and a rank-enhancing sine function.

\section{Related Work}

\subsubsection{Parameter Efficient Adapters}
Parameter efficient adaptation is a common fine-tuning strategy, in which a pre-trained base model is frozen, and a minimal number of adapter weights are trained on new data ~\cite{pmlr-v97-houlsby19a}. Low-Rank Adapters are a common variant, in which the adapter comprises two low-rank matrices ~\cite{hu2022lora}. VeRA ~\cite{kopiczko2024vera}, RandLoRA ~\cite{albert2025randlora}, and NOLA ~\cite{koohpayegani2024nola} use combinations of random projections to reduce the number of parameters contained within the adapters. QA-LoRA ~\cite{xu2024qalora} produces adapters that can be merged with the quantized base model, enabling low-precision inference. 

\subsubsection{Rank-Enhancing Functions} 
The most relevant rank-enhancing adapter works related to our approach are Ji et al. (2025) and Li et al. (2024) who investigate the use of sine-nonlinearities in low-rank adaptation ~\cite{ji2025efficient, li-etal-2024-loran}. We extend this approach by considering the effect of quantization on adapter performance.

\subsubsection{Delta Compression and Quantization}
Although adapters represent a trivial proportion of the total number of parameters in a network (typically less than 1\%), a recent branch of research has focused on the specific compression of these updates. Termed \textit{Delta Compression}, this branch recognizes the practical importance of reducing the memory throughput of fine-tuned updates, which may be distributed at scale to many parties with a common base model ~\cite{isik2023gptzip, yao2023deltazip, brüelgabrielsson2025compressserveservingthousands}. Within this framework it is typical to quantize the adapters, by mapping values to a limited set of floating points. This is often combined with lossless entropy compression such as zip. The quantization works most related to our approach are GPT-Zip ~\cite{isik2023gptzip}, Delta-Zip ~\cite{yao2023deltazip}, Bit Delta ~\cite{liu2024bitdelta}, and Bi-LoRA ~\cite{Jie_2023}. ~\cite{ping2024deltacome} use a mixed-precision strategy, devoting higher precision to larger singular values. ~\cite{jiang2024deltadqultrahighdeltacompression} uses a group-wise dropout and separate quantization. ~\cite{ryu2023efficientstoragefinetunedmodels} focus on low-rank residuals. ~\cite{liu2024bitdelta} uses binary adapters for Delta Compression. Our work differs from these models through our specific focus on the rank-increasing properties of a sine adaptation within a quantized framework.

\section{Theoretical Framework}\label{sec:theory}

\subsection{Preliminaries}\label{subsec:prelims}

\subsubsection{Sine-Activated Low-Rank Adapters}

\begin{figure}
    \centering
    \includegraphics[width=\linewidth]{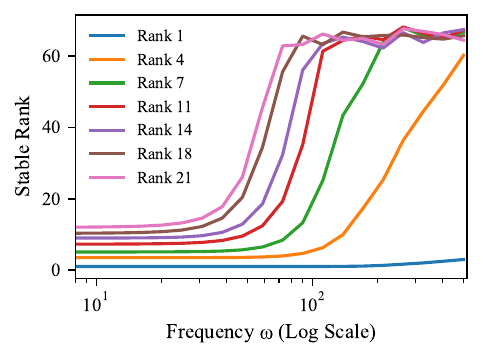}
    \caption{Effects of changing the $\omega$ term in a sine-activated low-rank matrix. Stable Rank saturates at sufficient $\omega$.}
    \label{fig:stable-rank-omega}
\end{figure}

Recent works by \cite{ji2025efficient} and \cite{li-etal-2024-loran} have explored the use of non-linear sine activations in adapter modules. Unlike common activations such as ReLU, sine functions can increase the rank of a matrix without adding additional parameters, offering a simple yet effective means of enhancing low-rank adapters. Specifically, \cite{ji2025efficient} introduced a sine-activated low-rank adapter of the form:

\begin{equation}
\frac{\sin(\omega AB)}{\gamma}
\end{equation}

where $\omega$ is a frequency parameter, $\gamma$ is a scaling factor, and $A \in \mathbb{R}^{m \times k}$, $B \in \mathbb{R}^{k \times n}$ are low-rank matrices with bottleneck dimension $k$.

\subsubsection{Stable Rank} The key insight of \cite{ji2025efficient} is that applying a sine function to the low-rank product $AB$, with large enough frequency $\omega$, 
would increase the stable rank of the matrix $AB$ which can then be used to increase the rank yielding a high rank adapter. The stable rank of a matrix $A$ is defined by: 
\begin{equation}
    \mathbf{SR}(A) := \frac{||A||_F^2}{(\sigma(A)_{\max})^2}
\end{equation}
where $||A||_F^2$ denotes the Frobenius norm and $\sigma_{\max}(A)$ the maximum singular value of $A$. Stable rank provides a softer measure of a matrix's effective dimensionality \cite{martinsson2020randomized}. Unlike the classical rank, which counts the number of nonzero singular values, the stable rank reflects how evenly the spectral energy is distributed. For instance, two matrices with identical rank can have vastly different stable ranks depending on the decay of their singular values. This nuance is critical when aiming to enhance low-rank adapters: even without increasing the classical rank, we can improve the adapter's expressivity by boosting its stable rank. This is precisely the property exploited by sine-activated adapters in \cite{ji2025efficient}. Figure \ref{fig:stable-rank-omega} shows how the stable rank is affected by changing the frequency term in a sine-activated matrix for different low-rank constraints.

\begin{figure}
    \centering
    \includegraphics[width=0.4\textwidth]{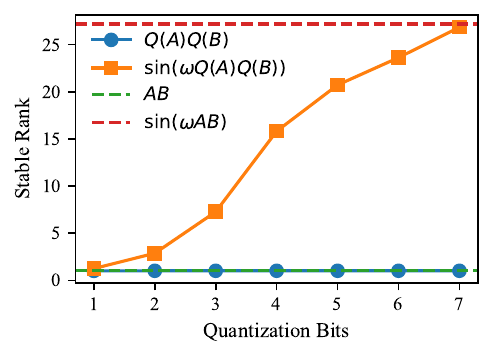}    

    \caption{A sine-activated low-rank matrix $\sin(\omega AB)$ increases the stable rank relative to a low-rank matrix $AB$. By varying quantization level $\sin(\omega(Q(A)Q(B))$ we can interpolate the effect on stable rank between these two values.}
    \label{fig:quantization_stable_rank}
\end{figure}

\subsubsection{Quantization}

A quantization function $Q(\cdot)$ maps values from a less restricted set to a more restricted set $\mathcal{A}\to\mathcal{B}$. Practically, this may involve explicit conversion of data-types (e.g. 16-bit precision to 4-bit precision), or maintaining the same data-type, but restricting the set of allowed values (e.g. mapping from $2^{16}$ discrete values to $2^4$) \cite{gholami2021surveyquantizationmethodsefficient, gray1998}. This type of quantization is common for memory compression and often coupled with an integer look-up table or an entropy coder \cite{han2015deep, Jacob_2018_CVPR}. It is conventional to define \textit{quantization error} as the residual resulting from a quantization map, which can be treated as a random variable \cite{gersho1991, gray1998}: 

\begin{equation}\label{eqn:quant_eqn}
    \epsilon = Q(A)-A
\end{equation}

For our experiments, we use a k-means quantization scheme due to its theoretical optimality and tractable implementation ~\cite{gersho1991, han2015deep}. This is implemented using the \textit{k-means1d} package ~\cite{Adam2019KMeans1D}, which provides an efficient wrapper for a fast k-means solver that runs in $O(kn + n \log n)$ for $n$ 1D data and $k$ clusters based on ~\cite{WU1991663,grønlund2018fastexactkmeanskmedians}. Further quantization experimental details are included in the Supplementary Materials. 

\subsection{Main Theorem}\label{subsec:main_theory}

In this section, we present our main theoretical result, which establishes that the stable rank of a quantized matrix is governed by the stable rank of its unquantized counterpart. We will use the notation $\sigma_{\max}$ to denote the maximum singular value of a matrix, $\sigma_{\min}$ to denote the minimum singular value and $\vert\vert \cdot\vert\vert_F$ to denote the Frobenius norm.

\begin{theorem}\label{thm:main_thm}
Let $A$ be a fixed matrix and let $Q$ denote a quantization operator so that $Q(A) = A - \epsilon$. Assume that $\sigma_{\min}(A) \leq 1$ and $\sigma_{\max}(A) >> \sigma_{\max}(\epsilon) >> 1$, so that
$\frac{\sigma_{\max}(A)}{2} \leq \sigma_{\max}(A) - \sigma_{\max}(\epsilon)$.
Then:
\begin{align}
\frac{1}{2}\bigg{(}
\sqrt{\mathbf{SR}(A)} - \frac{||\epsilon||_F}{\sigma_{\max}(A)}
\bigg{)}
&\leq 
\sqrt{\mathbf{SR}(Q(A))} \nonumber \\
&\leq 
2\bigg{(}
\sqrt{\mathbf{SR}(A)} + \frac{||\epsilon||_F}{\sigma_{\max}(A)}
\bigg{)}
\end{align} 
where $\epsilon$ is defined by \eqref{eqn:quant_eqn}.
\end{theorem}


\begin{proof}
We recall from \eqref{eqn:quant_eqn} that we can write: 
\begin{equation}
    Q(A) = A - \epsilon
\end{equation}
were $\epsilon$ is viewed as a random noise matrix. We then use the triangle inequality to obtain:
\begin{align}
\vert\vert A\vert\vert_F - \vert\vert \epsilon\vert\vert_F\leq \vert\vert Q(A)\vert\vert_F \leq \vert\vert A\vert\vert_F + \vert\vert \epsilon\vert\vert_F. \label{eqn:Q_A_triangle}
\end{align}
Using inequalities for the maximum singular value of a matrix we have:
\begin{align}
\sigma_{\max}(A) - \sigma_{\max}(\epsilon) &\leq   
\sigma_{\max}(Q(A)) \\
&\leq \sigma_{\max}(A) + \sigma_{\max}(\epsilon).\label{eqn:sing_ineq}
\end{align}
To prove the upper bound observe that: 
\begin{align}
    \sqrt{\mathbf{SR}(Q(A))} &= \frac{||Q(A)||_F}{\sigma_{max}(Q(A))} \\
    &\leq \frac{||A||_F + ||\epsilon||_F}{\sigma_{max}(Q(A))} \text{ by } \eqref{eqn:Q_A_triangle} \\
    &\leq \frac{||A||_F + ||\epsilon||_F}{\sigma_{max}(A) - \sigma_{\min}(A)} \text{ by } \eqref{eqn:sing_ineq} \\
    &\leq 2\bigg{(}\frac{||A||_F + ||\epsilon||_F}{\sigma_{max}(A)}\bigg{)}
\end{align}
where to get the last inequality we use the assumption in the statement of the theorem. The upper bound then follows from the definition of the stable rank.

To prove the lower bound we proceed in a similar way:
\begin{align}
\sqrt{\mathbf{SR}(Q(A))} &= \frac{||Q(A)||_F}{\sigma_{max}(Q(A))} \\
&\geq \frac{||A||_F - ||\epsilon||_F}{\sigma_{max}(Q(A))} \text{ by } \eqref{eqn:Q_A_triangle} \\
&\geq \frac{||A||_F - ||\epsilon||_F}{\sigma_{max}(Q(A)) + \sigma_{\max}(\epsilon)} \text{ by } \eqref{eqn:sing_ineq} \\
&\geq \frac{1}{2}\bigg{(}\frac{||A||_F - ||\epsilon||_F}{\sigma_{max}(Q(A))}\bigg{)}
\end{align}
where the last inequality comes from the assumption that
$\sigma_{\max}(A) \geq \sigma_{\max}(\epsilon)$. The lower bound then follows from the definition of stable rank.
\end{proof}

Theorem \ref{thm:main_thm} presents the key insight of this work: the stable rank of a quantized adapter remains low if the original (unquantized) adapter has low stable rank, as the quantized stable rank is controlled by the unquantized one. This observation motivates applying a sinusoidal function, with a large frequency $\omega$, after quantization. By leveraging results from \cite{ji2025efficient}, we note that a sine function with large frequency can increase the stable rank of the quantized adapter, effectively boosting its expressivity without sacrificing quantization efficiency. This produces a high-rank adapter while retaining the compression benefits of quantization. In particular this makes applying a sinusoidal function to a post quantization framework an effective way to yield better performance while still retaining compression benefits. Figure \ref{fig:quantization_stable_rank} provides an empirical illustration of our main insight. Starting with two low-rank matrices $A$ and $B$, whose product $AB$ is also low-rank, we apply quantization $Q$ to $A$ and $B$ at varying bits. The figure plots the stable ranks of $AB$, the quantized product $Q(A)Q(B)$, the sine-activated product $\sin(\omega AB)$, and $\sin(\omega Q(A)Q(B))$. As shown, the stable rank of $\sin(\omega Q(A)Q(B))$ increases with higher quantization bits, demonstrating how sinusoidal activation can effectively restore rank after quantization.


\subsection{Bjøntegaard Delta Analysis}

Bjøntegaard Delta (BD) Analysis is a commonly applied evaluation technique for comparing video and image compression codecs \cite{Bjontegaard2001, Herglotz2022, Herglotz_2024}, and has occasionally been applied for other modalities such as Point Cloud \cite{wang2021pointcloud, Wang_2021, Herglotz_2024,barman2022bjontegaard} or Neural Radiance Field compression \cite{ji2025efficient}. The metric involves evaluating two comparison codecs at two rate and performance positions. These are interpolated, with the measure evaluated as the integral between these positions. Figure \ref{fig:BD_visualisation} shows visually how the Bjøntegaard Delta can be calculated. Standard metrics include BD-Rate (the gain in performance at a given rate) and BD-PSNR (the gain in compression at a given performance). These can be extended without difficulty to performance evaluation metrics used for language models, such as average task accuracy (BD-Accuracy). 

\subsubsection{Mathematical Description}

Given RD points \( (R_i, D_i) \), we interpolate the RD curves using a smooth function \( f(R) \). The BD-rate is computed as:
\begin{equation}
    \Delta R = A \int_{D_{\min}}^{D_{\max}} \left( f_2^{-1}(D) - f_1^{-1}(D) \right) dD
\end{equation}
Similarly, BD-Accuracy is inversely defined:
\begin{equation}
    \Delta D = B\int_{R_{\min}}^{R_{\max}} \left( f_2(R) - f_1(R) \right) dR
\end{equation}

Where $A=\frac{1}{D_{\max} - D_{\min}}$ and $B=\frac{1}{R_{\max} - R_{\min}} $.

Typically, a cubic polynomial is used for $f(\cdot)$, however recent works have argued for the use of Akima interpolation due to its increased stability \cite{Herglotz2022,Herglotz_2024}. For our evaluation we use the \textit{Bjontegaard} Python library accessible at \url{https://github.com/FAU-LMS/bjontegaard}.

\subsubsection{Applicability to Delta Compression}
Model compression is an area of relatively recent interest in contrast to more established modalities such as images and video \cite{zhu2024surveymodelcompressionlarge, Siem2024LLMEdge}. As a result evaluating compression changes has yet to be standardised. It is therefore common to present model gains visually to show Pareto improvements; or to compare performance changes at an (approximately) equivalent parameter or performance position (e.g. `\dots performance with 25\% fewer parameters'). This leads to difficulties where parameters are not directly comparable or there is variation in algorithm performance at different memory levels. In contrast, the advantage of BD analysis is that it accounts for small inconsistencies in parameters, and accounts for the natural rate-performance trade-off that occurs during compression. We suggest that BD analysis can be used to evaluate Delta Compression performance.

\begin{figure}[h!]
    \centering
    \includegraphics[width=0.75\linewidth]{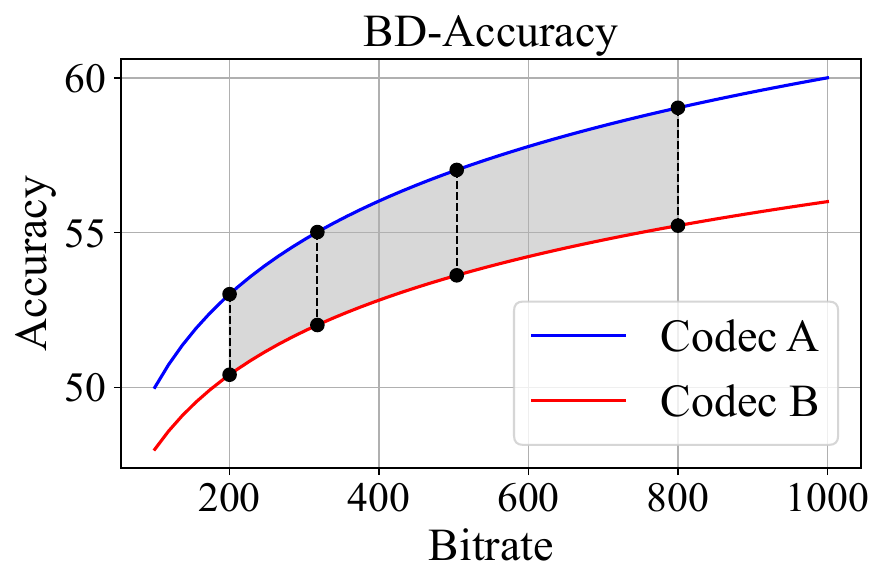}\\
    \includegraphics[width=0.75\linewidth]{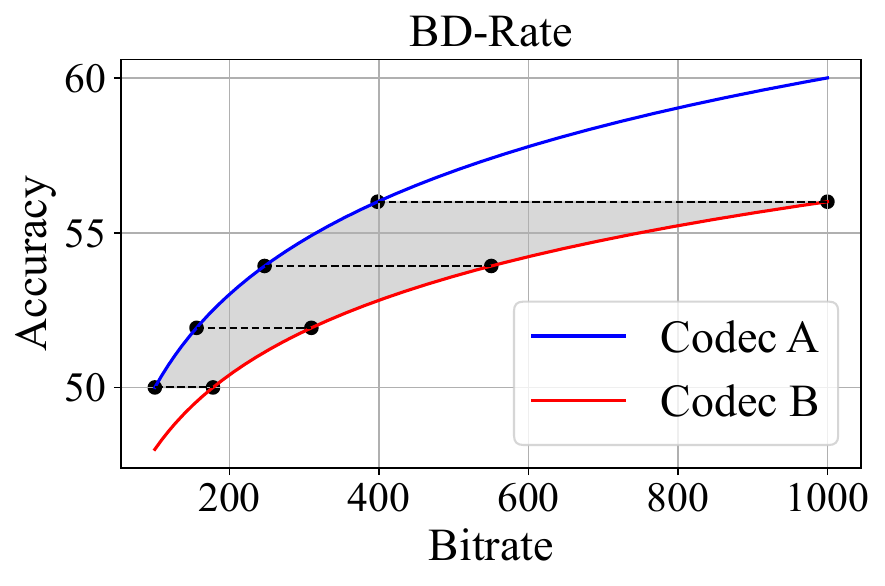}
    \caption{Bjøntegaard Delta is calculated by taking (bitrate, accuracy) pairs for two codecs, and evaluating the (horizontal or vertical) integral between interpolated performance.}
    \label{fig:BD_visualisation}
\end{figure}

\section{Results}\label{sec:results}
\subsection{Large Language Model Adaptation}

\subsubsection{Configurations} We fine-tune LLaMA 3-8B on a commonsense reasoning tasks, training the 15k dataset for 1 epoch. Following training we apply Post-Training Quantization at different bits-per-parameter, with each target tensor quantized independently. We then evaluate on each of the test sets directly, without further fine-tuning. We evaluate on a standard suite of benchmarks including BoolQ ~\cite{clark2019boolq}, PIQA ~\cite{bisk2019piqareasoningphysicalcommonsense}, SIQA ~\cite{sap2019socialiqa}, HellaSwag (HS) ~\cite{zellers2019hellaswag}, WinoGrande (WG) ~\cite{sakaguchi2021winogrande}, ARC-c, ARC-e ~\cite{clark2018think} and OBQA ~\cite{mihaylov2018can}. We use a frozen LLAMA-3-8B base model from Hugging Face ~\cite{llama3modelcard}. Each base experiment is run on one H100 GPU using a batch size of 128, and re-used for quantizing to different levels of precision. Low-rank adaptors are applied to the weight matrices $\mathbf{W}_\text{q},\mathbf{W}_\text{k},\mathbf{W}_\text{v},\mathbf{W}_\text{up},\text{and }\mathbf{W}_\text{down}$. We use the $\omega$ values in ~\cite{ji2025efficient}, who apply larger $\omega$ for low-rank models. Following ~\cite{ji2025efficient} we set $\gamma=\sqrt{n}$, where $n$ is the row dimension of the weight matrix. 

\subsubsection{Analysis.} Results are shown in Table \ref{tab:lora_sinelora_comparison}. We can note that the SineLoRA$\Delta$ model achieves significant memory compression and consistently outperforms LoRA. The memory reduction is such that the Rank 8 SineLoRA$\Delta$ at 5-bits outperforms the full-precision LoRA, with only $33.5\%$ of the memory (9.1MB to 27.1MB). Table \ref{tab:bd-llm} further validates this by examining the average performance improvements through Bjøntegaard Delta Analysis at each quantization level. At 2-bit quantization the SineLoRA$\Delta$ model shows an average $41.6\%$ memory improvement over LoRA, and an average accuracy improvement of 1.29\%. Full experimental results are recorded in Supplementary Tables 2 and 3, which ablates the effect of using a quantized base mode. Results are broadly comparable under this setting. We additionally compare results to DoRA \cite{liu2024dora}, and find that while both LoRA and SineLoRA$\Delta$ are robust to low (2-bit) quantization, the performance of DoRA degrades significantly until adapters are quantized to higher than 5-bit precision.

\begin{table}
    \centering
        
    \begin{tabular}{lccccc}
        \toprule
        & \multicolumn{5}{c}{Rank} \\
        \cmidrule(lr){2-6}
        Method & 1 & 2 & 4 & 8 & 16 \\
        \midrule
        LoRA (2-bit) & 69.7 & 71.0 & 74.7 & 75.2 & 77.3 \\
        SineLoRA$\Delta$ (2-bit) & \textbf{70.0} & \textbf{73.7} & \textbf{75.1} & \textbf{76.4} & \textbf{77.9} \\
        Memory (MB) & 0.6 & 1.1 & 2.2 & 4.3 & 8.6 \\
        \midrule
        LoRA (3-bit) & 70.0 & 73.1 & 75.5 & 76.5 & 78.4 \\
        SineLoRA$\Delta$ (3-bit) & \textbf{70.5} & \textbf{74.4} & \textbf{75.9} & \textbf{77.7} & \textbf{78.6} \\
        Memory (MB) & 0.8 & 1.5 & 3.0 & 6.0 & 11.9 \\
        \midrule
        LoRA (5-bit) & 69.4 & 73.1 & 75.6 & 76.7 & 78.6 \\
        SineLoRA$\Delta$ (5-bit) & \textbf{69.8} & \textbf{74.4} & \textbf{76.1} & \textbf{78.1} & \textbf{78.8} \\
        Memory (MB) & 1.2 & 2.3 & 4.5 & 9.1 & 18.1 \\
        \midrule
        LoRA (Full) & \textbf{73.7} & 74.8 & 76.5 & 78.0 & \textbf{79.0} \\
        SineLoRA$\Delta$ (Full) & 72.8 & \textbf{75.1} & \textbf{78.5} & \textbf{78.8} & 78.9 \\
        Memory (MB) & 3.4 & 6.8 & 13.5 & 27.1 & 54.0 \\
        \midrule 
        Parameters (M) & 1.8 & 3.5 & 7.1 & 14.2 & 28.3 \\
        \bottomrule
    \end{tabular} 

    \caption{Commonsense Reasoning performance for LoRA and SineLoRA$\Delta$ under different quantization rates. Averaged across tasks. Full refers to the typical 16-bit precision.}

    \label{tab:lora_sinelora_comparison}
\end{table}

\begin{figure*}
    \centering
    \includegraphics[width=0.32\linewidth]{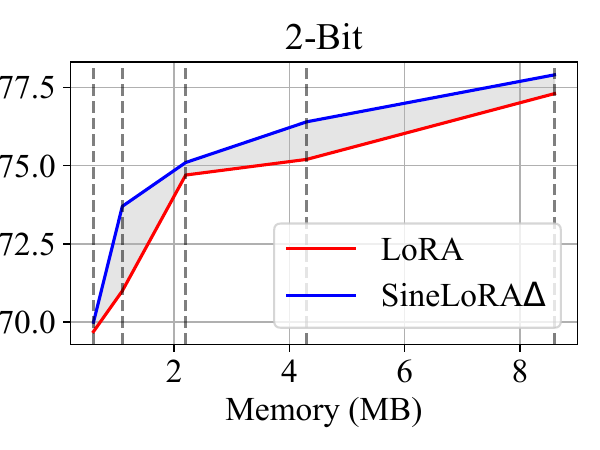}~
    \includegraphics[width=0.32\linewidth]{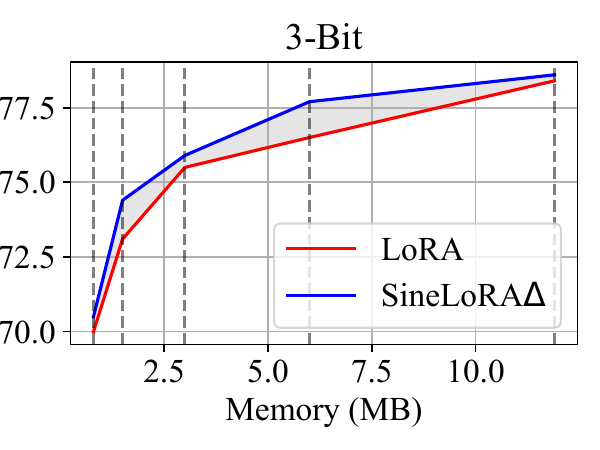}~
    \includegraphics[width=0.32\linewidth]{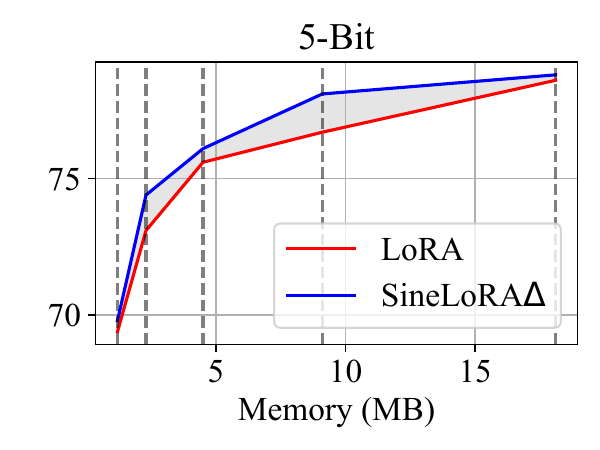}
    \caption{Commonsense Reasoning performance (average) for SineLoRA$\Delta$ and LoRA with a frozen non-quantized LLAMA-3-8B base model. SineLoRA$\Delta$ exceeds the benchmark LoRA performance across all evaluated rank and quantization levels. }
    \label{fig:area-under-curve}
\end{figure*}

\begin{table}[ht!]
    \centering
        
    \begin{tabular}{c|c|c}
         Quantization Level & BD-Rate $\downarrow$ & BD-Accuracy $\uparrow$  \\ \hline
        2 & -41.60\% & 1.29\% \\ 
        3 & -28.51\% & 0.88\% \\ 
        5 & -28.04\% & 0.96\% \\ 
        16 & -30.46\% & 0.69\% \\ \hline 
    \end{tabular}

    \caption{Bjøntegaard Delta Analysis for Table \ref{tab:lora_sinelora_comparison}, with the respective LoRA model as the baseline codec. Rate-distortion generated by keeping quantization fixed and varying the number of parameters through rank. SineLoRA$\Delta$ demonstrates improved performance at each quantization level.} 

    \label{tab:bd-llm}
\end{table}

\begin{table}[ht!]
    \centering
        
    \begin{tabular}{c|c|c}
         Quantization Level & BD-Rate $\downarrow$ & BD-Accuracy $\uparrow$  \\ \hline
        1 & 227.84\% & -3.49\% \\  
        2 & 19.89\% & -0.90\% \\ 
        3 & -15.65\% & 0.41\% \\ 
        4 & -40.74\% & 0.79\% \\ 
        5 & -42.81\% & 0.83\% \\ 
        8 & -47.56\% & 0.97\% \\ 
        16 & -44.38\% & 0.87\% \\ \hline 
    \end{tabular}

    \caption{Bjøntegaard Delta Analysis for Table \ref{tab:k-means_rank_comparison}, with the respective LoRA model as the baseline codec. Rate-distortion generated by keeping quantization fixed and varying the number of parameters through rank.}

    \label{tab:bd-vision}
\end{table}

\subsection{Vision-Language Model Adaptation}

\subsubsection{Data} We fine-tune CLIP ~\cite{clip-v139-radford21a} on 11 standard image classification datasets, obtained by following ~\cite{zhang2024knowledge}. These include: Cars ~\cite{cars2013}, DTD ~\cite{DTD2014}, EuroSAT ~\cite{eurosat2018}, Food101 ~\cite{bossard14}, Caltech101 ~\cite{caltech101_2006}, Sun397 ~\cite{xiao_sun_2016}, FGVCAircraft ~\cite{maji13fine-grained}, Flowers102 ~\cite{flowers102_2008}, ImageNet ~\cite{russakovsky_imagenet_2015}, Oxford Pets ~\cite{oxford_pets_2012}, and UCF101 ~\cite{soomro2012ucf101dataset101human}. We compare the performance of LoRA ~\cite{hu2022lora}, SineLoRA$\Delta$ ~\cite{ji2025efficient} for few-shot adaptation using a ViT-B/32 backbone following Post-Training Quantization.

\subsubsection{Configurations} Experiments are run on a NVIDIA GeForce RTX 4090 GPU with 24GB VRAM. Batch size 64, base model ViT-B/32, learning rate 0.001, weight decay 0.1, 10 epochs, AdamW optimizer ~\cite{loshchilov2018decoupled}. Fine-tuning is conducted on attention layers($\mathbf{W}_\text{q},\mathbf{W}_\text{k},\text{and } \mathbf{W}_\text{v}$) only. We finetune on few-shot tasks using 1 and 16 examples, employing different rank levels. We use $\omega = 200$ for all experiments, and $\gamma=\sqrt{n}$ where $n$ is the weight row dimension.

\subsubsection{Analysis} Tables \ref{tab:k-means_rank_comparison} and \ref{tab:bd-vision} shows our results on 1-shot classification averaged over 11 vision tasks. Consistent with the language model experiments, we observe that the SineLoRA$\Delta$ model outperforms the baseline LoRA at a similar rank and quantization. We observe that the SineLoRA$\Delta$ model only outperforms LoRA at 3-bits and higher, after which consistent compression improvements are found. This may be potentially explained by recalling Figure \ref{fig:quantization_stable_rank}, in which lower stable rank improvements are observed for very low precision (1 and 2 bit) quantization. In the Supplementary Materials we include additional ablations and comparison with DoRA \cite{liu2024dora}.

\begin{table*}
\centering

    \begin{tabular}{l c c c c c c c c c}
        \toprule
        Model & Rank & {1-Bit} & {2-Bit} & {3-Bit} & {4-Bit} & {5-Bit} & {8-Bit} & {Full} & {Params} \\
        \midrule
LoRA & 2 & \textbf{67.3} & \textbf{70.0} & 74.1 & 76.0 & 76.4 & 76.4 & 76.5 & 123K \\
SineLoRA$\Delta$ & 2 & 63.5 & 68.1 & \textbf{74.2} & \textbf{76.3} & \textbf{76.9} & \textbf{77.0} & \textbf{77.0} & 123K \\
\addlinespace
LoRA & 5 & \textbf{70.5} & \textbf{74.7} & 77.0 & 77.5 & 77.8 & 77.8 & 77.9 & 307K \\
SineLoRA$\Delta$ & 5 & 66.9 & 74.1 & \textbf{77.5} & \textbf{78.6} & \textbf{78.7} & \textbf{78.9} & \textbf{78.9} & 307K \\
\addlinespace
LoRA & 10 & \textbf{71.6} & \textbf{77.2} & 78.3 & 78.8 & 78.7 & 78.8 & 78.9 & 614K \\
SineLoRA$\Delta$ & 10 & 68.7 & 76.3 &\textbf{ 78.8} & \textbf{79.4} & \textbf{79.6} & \textbf{79.8} & \textbf{79.8} & 614K \\
\addlinespace
LoRA & 16 & \textbf{72.9} & \textbf{78.1} & 79.2 & 79.4 & 79.5 & 79.4 & 79.5 & 983K \\
SineLoRA$\Delta$ & 16 & 68.3 & 77.4 & \textbf{79.5} & \textbf{80.0} & \textbf{80.3} & \textbf{80.2} & \textbf{80.3} & 983K \\
\bottomrule
\end{tabular}
\caption{Vision-Language Model Adaptation (Averaged Over 11 Tasks) $\uparrow$}
\label{tab:k-means_rank_comparison}
\end{table*}

\subsection{Text-to-Image Generation} 

\subsubsection{Training Details}
To investigate how SineLoRA$\Delta$ performs on a text-to-image generation task, we adopt a DreamBooth fine-tuning pipeline \cite{ruiz2023dreambooth}. DreamBooth is a method for adapting text-to-image diffusion models using just a few reference images of a target object. Our experiments are performed on Stable Diffusion 3 Medium \cite{esser2024scaling}, using the official Hugging Face implementation\footnote{\url{https://github.com/huggingface/diffusers/tree/main/examples/dreambooth}}. For data, we use the DreamBooth dataset comprising 30 objects with 5-6 images per instance. For each object, we train a separate adapter. Following training, we quantize adapters to 1, 2, 3, and 5 bits using k-means quantization. These are evaluated using standard generative text prompts with 2 seeds each. For both LoRA and SineLoRA$\Delta$ we train rank 4 adapters for 300 epochs using the AdamW optimizer using a learning rate of $4\times 10^{-4}$ \cite{loshchilov2018decoupled}. For SineLoRA$\Delta$ we use a frequency $\omega=200$ and $\gamma=2\sqrt{n}$. All experiments are run on NVIDIA H100 GPUs, with each fine-tuning run taking around 7 minutes.

\begin{figure*}
    \centering
    \includegraphics[width=0.95\linewidth]{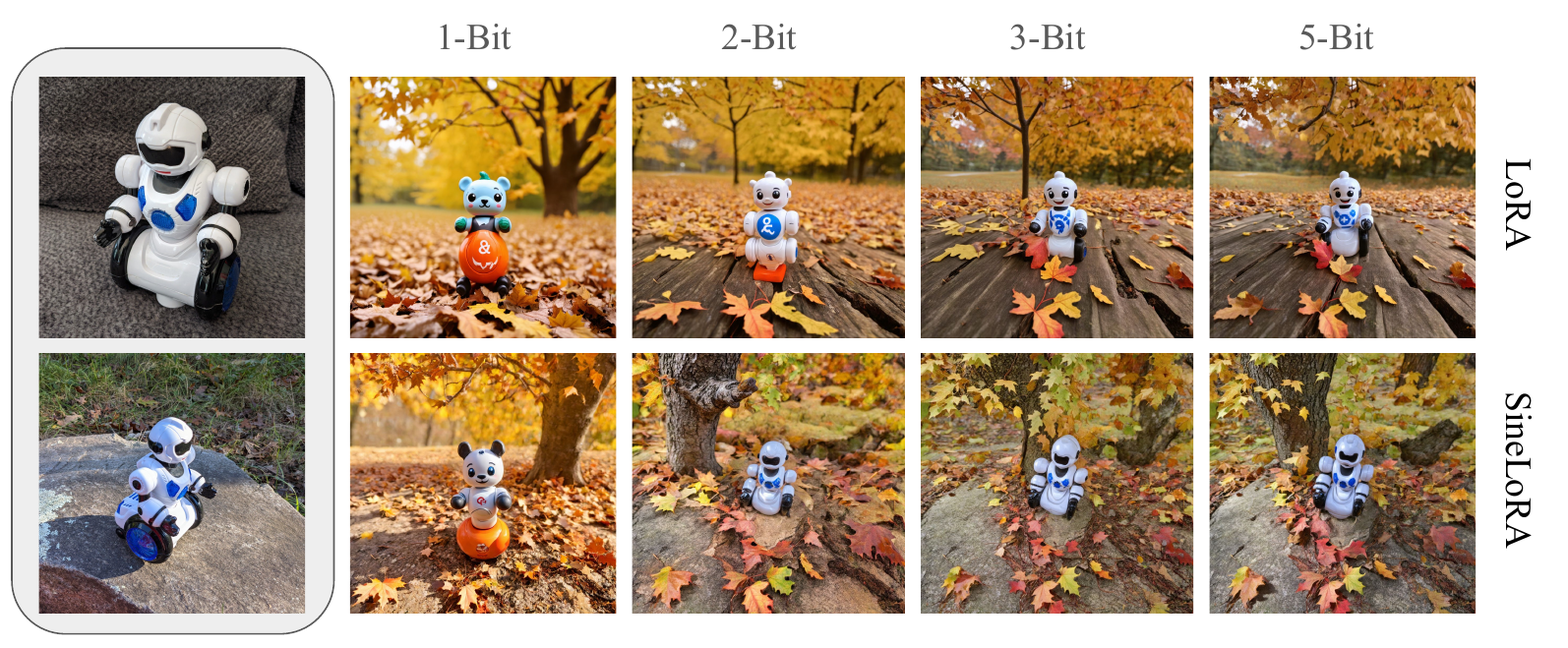}
  \caption{Dreambooth Stable Diffusion for the prompt \textbf{A toy with tree and autumn leaves in the background} for the category \textbf{robot toy}. SineLoRA$\Delta$ exhibits greater consistency with target images (left) than LoRA even at low levels of quantization.}
    \label{fig:robot_toy}
\end{figure*} 
\subsubsection{Analysis} Figure \ref{fig:robot_toy} shows a qualitative evaluation of SineLoRA$\Delta$ and LoRA trained using Dreambooth. Results show increased object fidelity for the SineLoRA$\Delta$ models, which is maintained at lower quantization levels than LoRA. Quantitatively, we follow \cite{ruiz2023dreambooth} and report the average cosine similarities between CLIP/DINO embeddings of generated images and subject images (CLIP-I and DINO), and of generated images and the text prompt (CLIP-T) \cite{clip-v139-radford21a, Caron_DINO_2021}. Table \ref{tab:stable-diffusion-clip-dino-scores} shows results averaged over all 30 categories evaluated at epoch 300. We find consistent performance improvements at each quantization level for CLIP-I and DINO, which measure the similarity to the target object. Evaluating the BD-Rate between LoRA and SineLoRA$\Delta$ we find a memory improvement of $-34.84\%$ on CLIP-I and $-29.48\%$ for DINO. Comparable performance between the two models is found CLIP-T, with a small 5\% improvement to the baseline LoRA. This is consistent with our qualitative results as CLIP-I and DINO measure how accurately the adapter has managed include the target object in the scene, while CLIP-T indicates how closely the overall scene matches the text prompt. We observe that 1-bit for both models has less fidelity to the fine-tuned target image, and appears dominated by the prompt. We attribute this to the increased dominance of the base model weights for generation. We provide additional qualitative results and analysis on individual category performance in the Supplementary Materials.

\begin{table}[t]
\centering

\begin{tabular}{llccc}
\toprule
\textbf{Bits} & \textbf{Model} & \textbf{CLIP-I $\uparrow$} & \textbf{CLIP-T $\uparrow$} & \textbf{DINO $\uparrow$} \\
\midrule
\multirow{2}{*}{1} 
& LoRA      & 0.729 & \textbf{0.219} & 0.515 \\
& SineLoRA$\Delta$  & \textbf{0.746} & \textbf{0.219} & \textbf{0.554} \\
\midrule
\multirow{2}{*}{2} 
& LoRA      & 0.768 & 0.218 & 0.599 \\
& SineLoRA$\Delta$  & \textbf{0.780} & \textbf{0.219} & \textbf{0.616} \\
\midrule
\multirow{2}{*}{3} 
& LoRA      & 0.780 & 0.218 & 0.621 \\
& SineLoRA$\Delta$  & \textbf{0.785} & \textbf{0.219} & \textbf{0.625} \\
\midrule
\multirow{2}{*}{5} 
& LoRA      & 0.783 & \textbf{0.219} & 0.626 \\
& SineLoRA$\Delta$  & \textbf{0.787} & \textbf{0.219} & \textbf{0.629} \\
\midrule
\multirow{2}{*}{Full} 
& LoRA      & 0.784 & \textbf{0.321} & 0.626 \\
& SineLoRA$\Delta$  & \textbf{0.790} & 0.317 & \textbf{0.632} \\
\bottomrule
\end{tabular}

\caption{Comparison of LoRA and SineLoRA$\Delta$ for Text-to-Image Generation. Best scores for each bit-width group and metric are highlighted in \textbf{bold}.}
\label{tab:stable-diffusion-clip-dino-scores}

\end{table}

\section{Discussion and Limitations}

\subsection{Quantization Aware Training}
Experimentally we have applied a Post-Training Quantization pipeline, which compresses weights following training. This has practical computational advantages as it allows evaluation of full rate-distortion curves without retraining at individual bit-rates. It is worth noting that improvements in performance are often possible by using Quantization Aware Training, which apply quantization during the training procedure \cite{gholami2021surveyquantizationmethodsefficient,rastegariECCV16}. While a systematic exploration of this is independent of our research question we note this as a direction for future research.

\subsection{Inference Precision}
The quantization scheme we have employed maps tensors to a restricted set of float values (e.g. $2^4$ values for 4-bit quantization), without recasting tensor data-types \cite{gholami2021surveyquantizationmethodsefficient}. This is commonly employed in memory compression for efficient data transfer. As both inference and training are conducted in the original data-type, it can be easily applied without modified memory types. However, this does not exploit GPU-level optimizations available for alternative data-types \cite{gholami2021surveyquantizationmethodsefficient, Dettmers_QLORA_23}. Combining our approach with methods such as QA-LoRA which enable INT-4 inference may lead to additional efficiency improvements \cite{xu2024qalora}.

\section{Conclusion}
\label{sec:conclusion}

In this work we have presented SineLoRA$\Delta$, a simple and effective enhancement for quantized low-rank adapters that significantly improves expressivity without introducing additional parameters. Our key theoretical insight is that stable rank under quantization is bounded by that of the original adapter - highlighting an inherent limitation of LoRA in compressed regimes. By applying fixed-frequency sinusoidal functions post-quantization, we demonstrate both analytically and empirically that these rank limitations can be mitigated. Across a diverse set of tasks - including large language model tuning, few-shot vision classification, and text-to-image generation - SineLoRA$\Delta$ delivers consistent improvements in accuracy at significantly reduced memory cost. We further propose the use of the Bjøntegaard Delta metric to evaluate compression-performance trade-offs in PEFT settings, providing a principled framework for comparing adapter methods along the rate-distortion curve. While our experiments focus on post-training quantization, the proposed method is modular and readily compatible with quantization-aware training or low-precision inference schemes. We view this work as a step toward scalable and bandwidth-efficient model adaptation, with potential applications in federated or multi-tenant deployment settings.

\appendix

\section*{Acknowledgments}
This research publication was supported in part by the CommBank Centre for Foundational AI Research.

\bigskip

\bibliography{aaai2026}

\clearpage

\includepdf[pages=-]{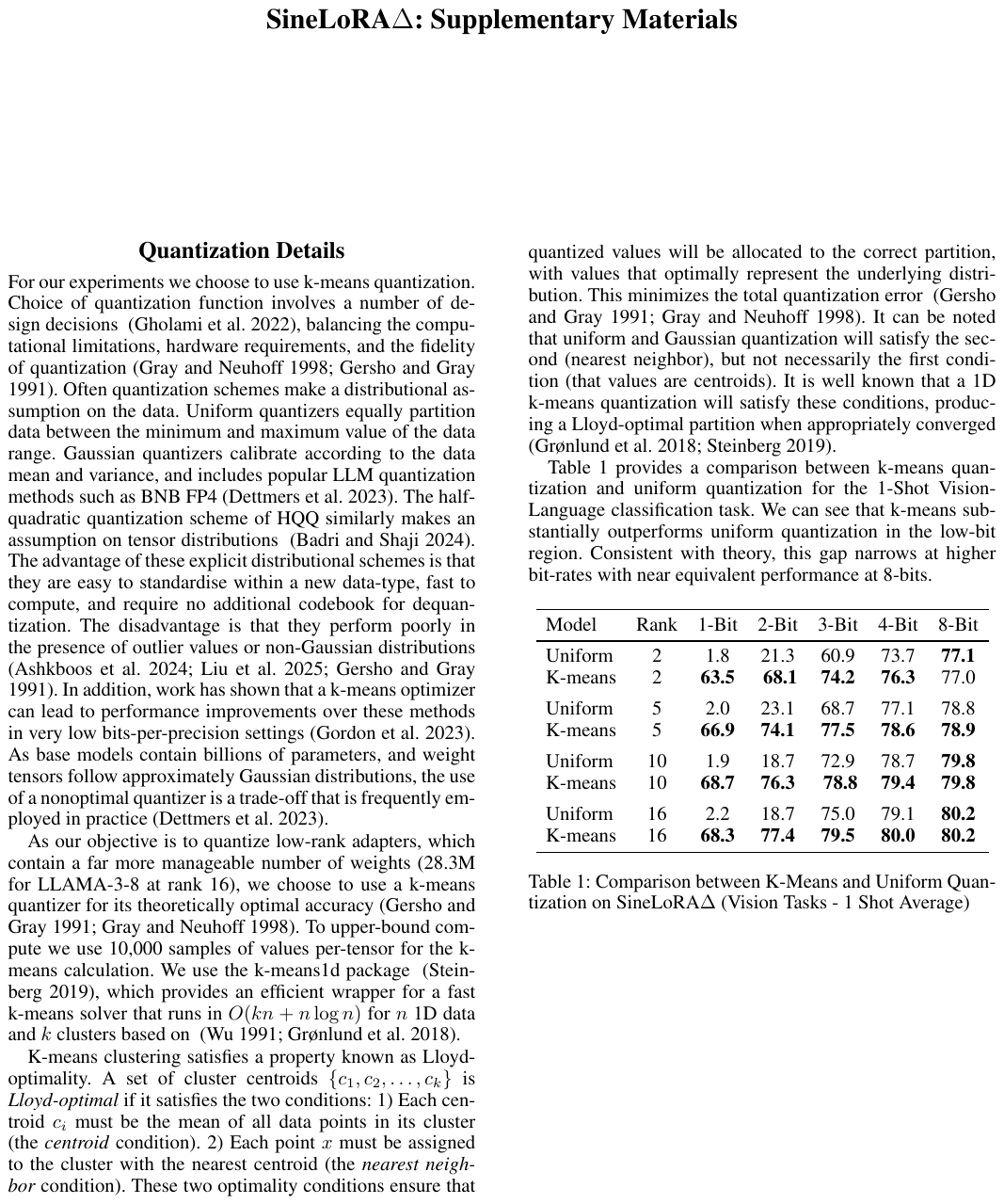}

\end{document}